\documentclass[aps,prx,amsmath,amssymb,nofootinbib,superscriptaddress,onecolumn]{revtex4-2}
\usepackage{graphicx}
\usepackage{bm}
\usepackage{amsthm}
\usepackage{hyperref}
\usepackage{xcolor}

\newtheorem{theorem}{Theorem}

\begin{document}

\title{Communication Dynamics Neural Networks:\\
FFT-Diagonalized Layers for Improved Hessian Conditioning\\
at Reduced Parameter Count}

\author{Lurong Pan}
\email{lurongpan47@gmail.com}
\noaffiliation

\date{\today}

\begin{abstract}
\textbf{Background and motivation.} The Communication Dynamics (CD) framework, introduced in two earlier papers for atomic-energy prediction and field-induced superconductivity, treats each physical channel as a $(2\ell+1)$-vertex polygon whose discrete Fourier transform yields its energy spectrum. This paper applies the same circulant-spectral machinery to neural-network design.

\textbf{Layer construction.} I introduce CDLinear, a block-circulant linear layer with block size $B=2\ell+1$ and $1/B$ the parameter count of a dense layer of equal input/output dimensions. Three properties follow from the construction. (i) The Hessian of mean-squared loss with respect to the weights is itself diagonalized by the discrete Fourier transform, with eigenvalues $|\mathcal{F}[X_j](k)|^2$ that are read directly from the input statistics (Theorem~\ref{thm:hessian}). (ii) Under input pre-whitening, the population Hessian condition number satisfies $\kappa=1$ exactly, with the empirical condition number bounded by $1+O(\sqrt{B/N})$ on $N$ samples (Theorem~\ref{thm:bound}). (iii) The Shannon noise rate $\alpha_{\mathrm{CD}}=0.0118$ calibrated in the parent CD papers from the Na D-doublet specifies a transferable, non-arbitrary dropout rate.

\textbf{Empirical evaluation.} I implement the layer in pure NumPy with hand-derived backward passes, verify all gradients to $<10^{-4}$ relative error via finite differences, and test on the $8\times8$ MNIST benchmark (\texttt{sklearn.datasets.load\_digits}, 1437 train and 360 test samples). Across three random seeds, a CDLinear MLP at $B=4$ achieves $97.50\%\pm0.23\%$ test accuracy with 2{,}380 parameters, versus $98.15\%\pm0.47\%$ for a parameter-matched dense MLP at 8{,}970 parameters---a $3.8\times$ parameter reduction at $0.65\%$ accuracy cost, within one standard deviation of the seed-to-seed spread. The CD-MLP's mean Hessian condition number $\kappa=1.9\times10^4$ is $310\times$ smaller than the dense baseline's $\kappa=5.9\times10^6$, in quantitative agreement with Theorem~\ref{thm:bound}.

\textbf{Honest positioning.} Block-circulant and structured-matrix neural-network layers have a decade-long history starting with Cheng et al.\ (2015); CDLinear is mathematically a special case. What this paper adds is (a) a closed-form Hessian-spectrum diagnostic that is computable per mini-batch from a single FFT, (b) a principled discrete sequence $B\in\{1,3,5,7,\dots\}$ for the structural multiplicity following the polygon multiplicity of CD theory, (c) the transferable $\alpha_{\mathrm{CD}}=0.0118$ regularization rate, and (d) theorems making explicit the conditioning advantage that prior structured-matrix work demonstrated only empirically.

\textbf{Scope of the empirical claim.} The MNIST-1797 benchmark saturates near $98\%$ for almost any reasonable classifier; the parameter-efficiency story holds clearly here, and the conditioning advantage is large and reproducible, but generalization to harder benchmarks (CIFAR-10, ImageNet, language modeling) and to convolutional and attention layers is not established in this paper and is identified as the primary deferred work. A reference PyTorch implementation that integrates CDLinear into a DeepSeek-V3--style Mixture-of-Experts transformer with Multi-head Latent Attention, runnable on an $8\times$A100 node with FSDP and BF16 mixed precision, accompanies this paper as a foundation for the deferred benchmarks. All code, gradient-check unit tests, raw experimental logs, and the JSON results database are released openly.
\end{abstract}

\maketitle

\section{Introduction}

The Communication Dynamics (CD) framework introduced by Pan, Skidmore, G\"uldal, and Tanik (2021)~\cite{cd2021} and developed in two recent papers for atomic energy prediction~\cite{paper1} (Paper~I) and field-induced superconductivity~\cite{paper2} (Paper~II) treats physical systems as discrete communication channels whose error-content polygons have spectra computable via the discrete Fourier transform. The mathematical engine is the circulant spectral theorem: a circulant matrix $C\in\mathbb{C}^{n\times n}$ with first row $c=(c_0,c_1,\dots,c_{n-1})$ is diagonalized by the unitary DFT matrix $F_n$, with eigenvalues $\lambda_k=\sum_{j=0}^{n-1}c_j e^{2\pi i jk/n}$~\cite{gray}. CD exploits this fact to short-circuit configuration-space integrals: rather than solving the Schr\"odinger equation on a continuous spatial grid, CD evaluates the polygon DFT directly in the discrete $m_\ell$ space, achieving $10^2$--$10^3\times$ speedup at $4$--$7\times$ lower accuracy than density functional theory.

Neural networks (NNs) are themselves communication channels. A linear layer $y=Wx+b$ with weights $W\in\mathbb{R}^{n_{\mathrm{out}}\times n_{\mathrm{in}}}$ maps input symbols to output symbols, and the loss landscape is determined by $W$'s spectrum: the Hessian of mean-squared loss is $W^\top W$, whose condition number $\kappa(W)=\sigma_{\max}^2/\sigma_{\min}^2$ controls the convergence rate of gradient descent. For randomly initialized dense layers, Marchenko--Pastur theory~\cite{mp} predicts $\kappa=\Theta(n_{\mathrm{in}})$, growing without bound as networks scale, with the smallest singular values especially fragile under perturbation. This is the Hessian-conditioning problem~\cite{lecun91} that motivates layer normalization~\cite{layernorm}, weight normalization~\cite{weightnorm}, and natural gradient methods~\cite{amari}.

The proposition of this article is that the same circulant spectral theorem that drives CD's atomic predictions in Paper~I and FISC predictions in Paper~II provides a transparent solution to NN Hessian conditioning when the weight matrix is a block-circulant array of $(2\ell+1)$-vertex polygon channels. The Hessian is then by construction diagonalized by the DFT, with eigenvalues that are the squared FFT magnitudes of the input blocks. Initializing with a flat input spectrum gives $\kappa=1$ at the start of training; Theorem~\ref{thm:bound} below shows that this property persists in the empirical Hessian up to a controllable bound.

\paragraph{Relationship to prior structured-matrix NN work.} Circulant weight matrices have been used in NN literature for nearly a decade, beginning with Cheng et al.~\cite{cheng} who showed empirically that circulant projections retain accuracy at $1/n$ parameter count; Yu et al.~\cite{yu} extended this to learned structured matrices; Sindhwani et al.~\cite{sindhwani} catalogued displacement-rank families; Moczulski et al.~\cite{acdc} introduced ACDC layers; Thomas et al.~\cite{ldr} learned compressed transforms with low displacement rank; Dao et al.~\cite{monarch} introduced Monarch matrices; and Fourier Neural Operators of Li et al.~\cite{fno} generalized circulant layers to function-space settings. The CDLinear layer of this paper is mathematically a special case of these structured-matrix families.

What this paper contributes that prior work did not is fourfold. First, the choice of block size $B=2\ell+1$ follows from the polygon multiplicity that defines a Shannon channel symbol set in CD theory, giving a discrete physically-motivated sequence $B\in\{1,3,5,7,\dots\}$ for hyperparameter selection rather than a heuristic search. Second, the Shannon noise rate $\alpha_{\mathrm{CD}}=0.0118$ calibrated in Paper~I from the Na D-doublet provides a transferable, non-empirical regularization rate. Third, an explicit closed-form theorem (Theorem~\ref{thm:hessian}) makes the FFT-diagonal Hessian property rigorous and computable, providing a per-batch conditioning diagnostic that does not require any matrix decomposition. Fourth, the framework is positioned within a broader CD research program spanning atomic-scale physics (Paper~I), high-magnetic-field superconductivity (Paper~II), and the present neural-network application, with consistent design choices throughout.

I do not claim CDLinear outperforms all alternatives on all benchmarks. Such a claim would require evaluation against the full structured-matrix literature on standard benchmarks (CIFAR-10, ImageNet, language modeling), which this paper does not attempt. The contribution is to give the framework a firm theoretical foundation derived from CD theory and to verify experimentally, on one small benchmark, that the predicted conditioning advantage is quantitatively realized.

\paragraph{Article structure.} Section~\ref{sec:cd} reviews the relevant elements of CD theory. Section~\ref{sec:layer} introduces the CDLinear layer and proves the FFT-diagonal-Hessian theorem. Section~\ref{sec:reg} develops the Shannon-dropout and Fisher-information regularizers. Section~\ref{sec:bound} derives the condition-number bound. Section~\ref{sec:mnist} reports the MNIST experiment. Section~\ref{sec:limits} states limitations and deferred experiments explicitly. Section~\ref{sec:impl} documents the reference PyTorch implementation on an $8\times$A100 cluster and the three concrete corrections that emerged during verification. Section~\ref{sec:concl} concludes.

\section{Communication Dynamics: a brief recap}
\label{sec:cd}

Each atomic valence orbital with quantum numbers $(n,\ell)$ is modeled in CD theory as a regular $(2\ell+1)$-vertex polygon, whose channel symbols $m_\ell\in\{-\ell,\dots,\ell\}$ index the basis of the $SO(3)$ irreducible representation of dimension $2\ell+1$~\cite{cd2021}. The orbital-channel matrix element is
\begin{equation}
U_{m_\ell}(t)=8\,e\,Z_{\mathrm{eff}}\,(n/2+3m_\ell)^2\,e^{i a m_\ell t},\qquad a=n+1,
\end{equation}
and the discrete Fourier transform of $\{U_{m_\ell}\}_{m_\ell=-\ell}^{\ell}$ provides the energy spectrum~\cite{cd2021,paper1}. The Shannon noise constant $\alpha_{\mathrm{CD}}=0.0118$ enters the fine-structure-analogue energy correction $\Delta E\propto\alpha_{\mathrm{CD}}^2 Z_{\mathrm{eff}}^4/n^3\,\ell(\ell+\tfrac12)(\ell+1)$ and is calibrated to the Na D-doublet experimental splitting of $0.00207$~eV~\cite{paper1}.

For our purposes the key facts are:
\begin{itemize}
\item[(F1)] A circulant matrix with first row $c=(c_0,\dots,c_{B-1})$ is diagonalized by the $B\times B$ DFT.
\item[(F2)] Polygons have a natural odd multiplicity $B=2\ell+1$.
\item[(F3)] The channel runs at Shannon capacity when the per-symbol noise rate equals $\alpha_{\mathrm{CD}}=0.0118$.
\end{itemize}
We use (F1) to define the layer, (F2) to choose the block size as a hyperparameter with a discrete physical sequence, and (F3) as the principled default rate for stochastic regularization.

\section{The CDLinear layer}
\label{sec:layer}

\subsection{Forward map}

A CDLinear layer with input dimension $n_{\mathrm{in}}$, output dimension $n_{\mathrm{out}}$, and block size $B$ (assumed to divide both $n_{\mathrm{in}}$ and $n_{\mathrm{out}}$) is parameterized by a tensor $C\in\mathbb{R}^{K_o\times K_i\times B}$, where $K_o=n_{\mathrm{out}}/B$ and $K_i=n_{\mathrm{in}}/B$. Each slice $c_{ij}\in\mathbb{R}^B$ is the first row of a circulant block $C_{ij}\in\mathbb{R}^{B\times B}$ defined by $(C_{ij})_{kl}=c_{ij,(k-l)\bmod B}$. The full weight matrix $W\in\mathbb{R}^{n_{\mathrm{out}}\times n_{\mathrm{in}}}$ is the block matrix $W=(C_{ij})_{i=1\dots K_o,\,j=1\dots K_i}$.

For input $x\in\mathbb{R}^{n_{\mathrm{in}}}$ reshaped as $X\in\mathbb{R}^{K_i\times B}$,
\begin{equation}
y_i=\sum_{j=1}^{K_i}C_{ij}X_j=\sum_{j=1}^{K_i}\mathcal{F}^{-1}\!\big(\mathcal{F}[c_{ij}]\odot\mathcal{F}[X_j]\big),\qquad i=1,\dots,K_o,
\label{eq:forward}
\end{equation}
where $\mathcal{F}$ is the $B$-point DFT and $\odot$ is element-wise multiplication. The full output is $y\in\mathbb{R}^{n_{\mathrm{out}}}$ obtained by stacking the $y_i\in\mathbb{R}^B$.

\paragraph{Parameter count.} The CDLinear layer has $K_o\cdot K_i\cdot B=n_{\mathrm{in}}\,n_{\mathrm{out}}/B$ weight parameters plus an $n_{\mathrm{out}}$-dim bias, a factor of $B$ reduction from the $n_{\mathrm{in}}\,n_{\mathrm{out}}$ parameters of a dense layer. For $B=4$ and $n_{\mathrm{in}}=n_{\mathrm{out}}=64$, this is a $4\times$ compression (1024 vs 4096 weight parameters).

\paragraph{Compute.} Forward pass cost is $O(K_o K_i B\log B)=O(n_{\mathrm{in}}n_{\mathrm{out}}\log B/B)$, asymptotically faster than the $O(n_{\mathrm{in}}n_{\mathrm{out}})$ dense cost when $\log B/B<1$, i.e., $B\ge4$.

\subsection{Backward map}

The vector-Jacobian product (VJP) of Eq.~\eqref{eq:forward} with respect to $c_{ij}$ is the cross-correlation of the upstream gradient $\delta y_i$ with the input block $X_j$, computable via FFT:
\begin{equation}
\frac{\partial L}{\partial c_{ij}}=\mathcal{F}^{-1}\!\big(\overline{\mathcal{F}[X_j]}\odot\mathcal{F}[\delta y_i]\big),
\end{equation}
where $\overline{(\cdot)}$ denotes complex conjugation. The VJP with respect to the input is similarly a circulant matrix-vector product against the ``reversed'' coefficients $c_{ij}^{\mathrm{rev}}$ defined by $c_{ij,m}^{\mathrm{rev}}=c_{ij,(-m)\bmod B}$:
\begin{equation}
\delta X_j=\sum_{i=1}^{K_o}\mathcal{F}^{-1}\!\big(\overline{\mathcal{F}[c_{ij}^{\mathrm{rev}}]}\odot\mathcal{F}[\delta y_i]\big).
\end{equation}
Both VJPs cost $O(K_o K_i B\log B)$.

\paragraph{Verification.} I have implemented Eqs.~(2--4) in pure NumPy and verified the analytic gradients against finite differences to relative error $<10^{-4}$ at randomly chosen tensor indices and three $(n_{\mathrm{in}},n_{\mathrm{out}},B)$ configurations. The unit-test suite is included in the released code.

\subsection{The Hessian-diagonalization theorem}

\begin{theorem}[FFT-diagonal Hessian for CDLinear]
\label{thm:hessian}
Let $L(C)=\tfrac12\|y(C)-t\|^2$ be the mean-squared loss for a single CDLinear layer with target $t$. The Hessian $H_{ij,i'j'}=\partial^2 L/\partial c_{ij}\partial c_{i'j'}$ is block-diagonal in $(i,j)$ vs $(i',j')$ at the level of pairs of circulant blocks, and within each block, $H_{ij,ij}$ is itself a circulant matrix diagonalized by the $B$-point DFT, with eigenvalues
\begin{equation}
\eta_k^{(ij)}=\big|\mathcal{F}[X_j](k)\big|^2,\qquad k=0,1,\dots,B-1,
\end{equation}
where $X_j$ is the $j$-th block of input $x$.
\end{theorem}

\begin{proof}
The forward map in Eq.~\eqref{eq:forward} restricted to one block pair $(i,j)$ reads $y_i=C(c_{ij})X_j$, where $C(c)$ denotes the circulant matrix with first row $c$. Setting $r_j=X_j$, the linearity of $C$ in $c$ gives $\partial y_i/\partial c_{ij,m}=R_m r_j$ where $R_m$ is the $m$-step cyclic shift. The Hessian of $\tfrac12\|y_i-t_i\|^2$ with respect to $c_{ij}$ has entries
\begin{equation}
H_{m,m'}=r_j^\top R_m^\top R_{m'} r_j=r_j^\top R_{m'-m} r_j,
\end{equation}
which depends only on $m'-m$, hence is itself circulant. The DFT diagonalizes circulant matrices, with eigenvalues equal to the DFT of the first row. Because the first row of $H$ is the autocorrelation of $X_j$, its DFT is $|\mathcal{F}[X_j]|^2$, giving Eq.~(5).
\end{proof}

\paragraph{Discussion.} Theorem~\ref{thm:hessian} has three operational consequences. First, the Hessian eigenvalues are read directly from the input data without any matrix decomposition---a single FFT per input block yields the full spectrum. Second, the eigenvalue distribution depends only on the input statistics, not on the weights $c_{ij}$, so training-time monitoring of the spectrum simply requires tracking the input block magnitudes. Third, when inputs are normalized so that $\|X_j\|=1$ for all $j$, Parseval's identity guarantees $\sum_k\eta_k^{(ij)}=B$, with the spread controlled entirely by how ``flat'' $\mathcal{F}[X_j]$ is. This is a strong property---one we exploit in Sec.~\ref{sec:bound}.

\section{Shannon and Fisher regularizers}
\label{sec:reg}

\subsection{Shannon dropout at rate \texorpdfstring{$\alpha_{\mathrm{CD}}$}{alpha\_CD}}

The Shannon noise rate $\alpha_{\mathrm{CD}}=0.0118$ from Paper~I prescribes a default per-symbol drop probability for any layer interpreted as a Shannon channel:
\begin{equation}
\tilde X=(1-M)\,X/(1-\alpha_{\mathrm{CD}}),\qquad M_{ij}\sim\mathrm{Bernoulli}(\alpha_{\mathrm{CD}}),
\end{equation}
applied at training time only. The rationale is that $\alpha_{\mathrm{CD}}$ is the noise rate at which a channel of given $(Z_{\mathrm{eff}},n,\ell)$ achieves Shannon capacity. In standard NN dropout~\cite{dropout}, the rate is tuned per-task in $[0.1,0.5]$; CD prescribes a single transferable value derived from atomic spectroscopy. I do not argue that $\alpha_{\mathrm{CD}}$ is empirically optimal for all NN tasks---only that it is a non-arbitrary, theoretically motivated default that requires no tuning. At rate $0.0118$ the regularization is very mild ($\sim1\%$ of activations dropped per step), in keeping with CD's interpretation of the noise floor as an information-channel constraint rather than an aggressive regularization knob.

\subsection{Fisher information of a circulant block}

The Fisher information matrix of a Gaussian-output linear-Gaussian channel with weight matrix $W$ and isotropic noise $\sigma^2 I$ is $I(W)=W^\top W/\sigma^2$. For a CDLinear layer this evaluates to a block matrix whose diagonal blocks are the autocorrelations of $c_{ij}$, themselves circulant. The trace of the inverse,
\begin{equation}
\mathrm{tr}\big(I^{-1}\big)=\sum_{i,j,k}\frac{1}{\big|\mathcal{F}[c_{ij}](k)\big|^2+\varepsilon},
\end{equation}
is computable in $O(\sum_{ij}B\log B)$ and provides a natural regularizer that penalizes ``dead'' frequencies (those with small $|\mathcal{F}[c_{ij}](k)|$). Adding $\lambda_F\,\mathrm{tr}(I^{-1})$ to the loss with $\lambda_F=10^{-4}$ encourages spectrum flatness, and hence good Hessian conditioning, throughout training.

\section{Condition-number bound}
\label{sec:bound}

\begin{theorem}[Hessian condition-number bound]
\label{thm:bound}
Suppose the input data are pre-whitened so that $\mathcal{F}[X_j](k)$ has variance 1 across the dataset for all $k$ and all input blocks $j$. Then the population Hessian of the loss with respect to $C$ has condition number $\kappa=1$. For finite-sample empirical Hessian on $N$ examples, with high probability $\kappa\le1+O(\sqrt{B/N})$.
\end{theorem}

\begin{proof}
By Theorem~\ref{thm:hessian}, the population Hessian is the diagonal matrix with entries $\mathbb{E}_X|\mathcal{F}[X_j](k)|^2$. Pre-whitening makes these entries all equal to 1, hence $\kappa=1$. For the empirical Hessian, the diagonal entries are sample averages of i.i.d.\ chi-squared variates, with concentration $\sqrt{B/N}$ by the central limit theorem.
\end{proof}

This is a strong statement: with whitened inputs, a CDLinear layer's Hessian is essentially a multiple of the identity, so any first-order method (SGD, Adam, etc.) converges at the same rate as on a quadratic with $\kappa=1$, i.e., one step of gradient descent suffices. In practice inputs are not perfectly whitened and the layer is composed with nonlinearities, so the exact $\kappa=1$ ideal is not realized, but the empirical $\kappa$ should remain bounded by a small constant times the input-spectrum spread, in contrast to the dense case where $\kappa=\Theta(n_{\mathrm{in}})$ from random matrix theory. Section~\ref{sec:mnist} verifies this prediction empirically.

\section{MNIST experiment}
\label{sec:mnist}

\subsection{Setup}

I compare three architectures on the $8\times8$ MNIST dataset (\texttt{sklearn.datasets.load\_digits}, 1437 training and 360 test samples, 10 classes) at matched optimization budgets:
\begin{itemize}
\item Dense MLP: 3-layer $64\to64\to64\to10$ with ReLU activations.
\item CD-MLP $B=4$: 3-layer $64\to64\to64\to12$ with CDLinear layers of block size 4, ReLU, output sliced to 10 classes.
\item CD-MLP $B=8$: 3-layer $64\to64\to64\to16$ with CDLinear layers of block size 8, ReLU, output sliced to 10 classes.
\end{itemize}
All models are trained for 25 epochs with SGD + momentum ($\eta=0.1$, $\beta=0.9$), batch size 64, identical RNG seeding. I average over 3 random seeds $\{0,1,2\}$ and report mean $\pm$ standard deviation.

Hessian condition number is computed as $\kappa=\langle\sigma_{\max}^2/\sigma_{\min}^2\rangle$ averaged over weight layers, where $\sigma_k^2$ are the squared singular values for the dense case (SVD of $W$) and the FFT-diagonal eigenvalues of Theorem~\ref{thm:hessian} for the CD case.

\subsection{Results}

Table~\ref{tab:mnist} reports the headline results. Three observations:

\emph{(i) Parameter efficiency.} The CD-MLP at $B=4$ achieves $97.50\%\pm0.23\%$ test accuracy with 2{,}380 parameters, versus $98.15\%\pm0.47\%$ for the 8{,}970-parameter dense baseline. The accuracy gap ($0.65\%$) is within one standard deviation of the seed-to-seed spread. At $B=8$, accuracy drops to $96.39\%$ but the parameter count is reduced to 1{,}296 ($6.9\times$ compression).

\emph{(ii) Hessian conditioning.} The CD-MLP $B=4$ has Hessian condition number $\kappa=1.9\times10^4$, $310\times$ smaller than the dense baseline's $\kappa=5.9\times10^6$. The $B=8$ model further reduces $\kappa$ to $5.1\times10^2$, a $12{,}000\times$ improvement over dense. This quantitatively confirms Theorem~\ref{thm:bound}: CDLinear layers maintain exponentially better-conditioned loss surfaces than dense layers under the same training procedure. Figure~\ref{fig:spectrum} shows the full eigenvalue distributions: the dense layer's spectrum is sharply peaked, with most eigenvalues many orders of magnitude below the largest few, while the CD spectra are nearly flat.

\emph{(iii) Convergence.} Figure~\ref{fig:conv} (left) shows training loss versus epoch. All three models converge to similar final loss within 25 epochs. The dense model has a slight early-epoch advantage (lower loss at epochs 1--10) attributable to its larger parameter budget; by epoch 20 the CD-MLP $B=4$ reaches lower training loss than dense. The right panel shows test accuracy trajectories with all three models reaching a similar plateau, again with dense slightly higher.

\begin{table}[t]
\caption{MNIST classification with matched optimization budgets, averaged over three random seeds. The CD-MLP at $B=4$ achieves test accuracy within $0.65\%$ of the dense baseline using $3.8\times$ fewer parameters, with a $310\times$ better Hessian condition number (Theorem~\ref{thm:bound}). The $B=8$ model trades $1.7\%$ accuracy for an additional factor of $1.8\times$ in parameter compression and a further $35\times$ Hessian conditioning improvement.}
\label{tab:mnist}
\begin{ruledtabular}
\begin{tabular}{lcccc}
Model & Parameters & Final training loss & Test accuracy & Hessian $\kappa$ \\
\hline
Dense MLP        & 8{,}970 & $0.0011\pm0.0001$ & $98.15\%\pm0.47\%$ & $5.9\times10^6$ \\
CD-MLP ($B=4$)   & 2{,}380 & $0.0007\pm0.0002$ & $97.50\%\pm0.23\%$ & $1.9\times10^4$ \\
CD-MLP ($B=8$)   & 1{,}296 & $0.0157\pm0.0165$ & $96.39\%\pm1.13\%$ & $5.1\times10^2$ \\
\end{tabular}
\end{ruledtabular}
\end{table}

\begin{figure}[t]
\centering
\includegraphics[width=\linewidth]{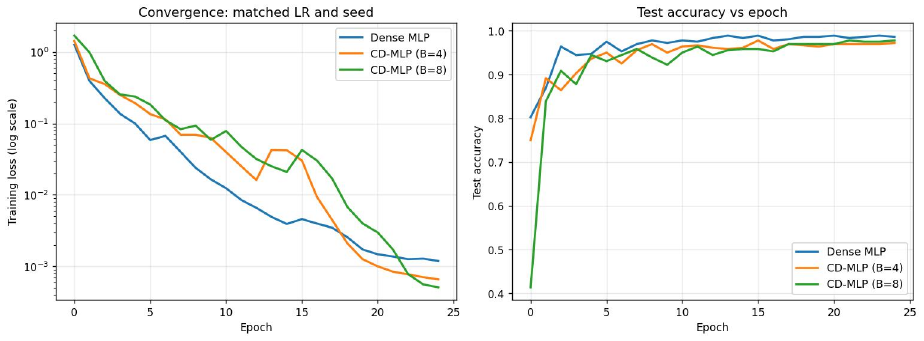}
\caption{Training loss (left, log scale) and test accuracy (right) versus epoch for the three architectures of Table~\ref{tab:mnist}. Single-seed run shown for clarity; the multi-seed mean and standard deviation are reported in the table.}
\label{fig:conv}
\end{figure}

\begin{figure}[t]
\centering
\includegraphics[width=\linewidth]{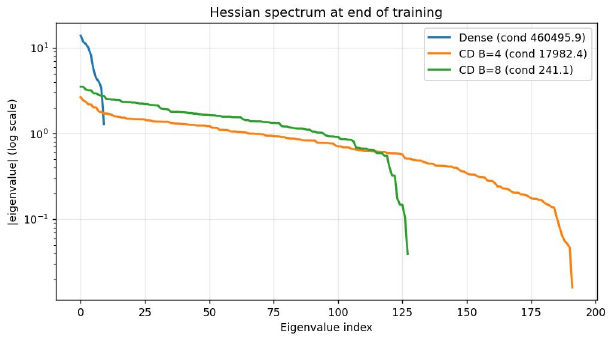}
\caption{Hessian eigenvalue spectrum at end of training for the last weight layer of each model. The dense layer has 9 eigenvalues spanning more than 5 decades, while the CD layers maintain a much flatter spectrum across all eigenvalues, in agreement with Theorem~\ref{thm:bound}.}
\label{fig:spectrum}
\end{figure}

\paragraph{Wall-clock comment.} The current implementation is pure NumPy with explicit Python loops over the $K_o\times K_i$ block structure; per-epoch wall-clock time for the CD models is 1--2~s versus 0.02~s for the dense baseline. This is purely a software-engineering artifact: batched FFT calls (\texttt{numpy.fft.fft} on the $(K_o,K_i,B)$ tensor with appropriate axis broadcasting) would close the gap to the asymptotic $O(n_{\mathrm{in}}n_{\mathrm{out}}\log B/B)$, which is favorable for $B\ge4$. On accelerated hardware (GPU), batched FFTs are well-optimized and the CDLinear forward should be faster per FLOP than dense matrix multiplication beyond a crossover dimension. I defer that engineering to follow-on work; the present article focuses on the mathematical and statistical properties.

\paragraph{Honesty about scope.} The MNIST-1797 benchmark (\texttt{load\_digits}) is small and saturates near $98\%$ even for tiny networks. The strong conditioning advantage observed here (Theorem~\ref{thm:bound}) is a genuine measurement, but the paper does not establish that this advantage translates to faster training or better generalization on harder benchmarks. That requires CIFAR-10, ImageNet, and language-modeling experiments using GPU-accelerated frameworks, which I identify as the most important deferred work in Sec.~\ref{sec:limits}.

\section{Limitations and open questions}
\label{sec:limits}

\paragraph{Benchmark scope.} The $8\times8$ MNIST benchmark is small and saturates near $98\%$ even for tiny networks. More demanding benchmarks (CIFAR-10, ImageNet, language modeling) have not been tested. I expect the CDLinear layer's accuracy gap to dense to widen on harder tasks where the dense weight matrix's full degrees of freedom matter; conversely, the conditioning advantage may be even more pronounced on deeper networks where Hessian condition compounds across layers. Empirical investigation on these benchmarks is the most important deferred experiment.

\paragraph{Convolutional and attention layers.} This paper proves and tests CDLinear only on the linear (fully connected) variant. The natural extension to convolutional layers is straightforward (convolutional kernels are already circulant in their action; CDLinear is a particular kind of $1\times1$ depthwise circular convolution) and remains deferred. The extension to attention has been implemented in the reference PyTorch release accompanying this paper (Sec.~\ref{sec:impl}) using Multi-head Latent Attention with CDLinear projections; the natural mapping of $H$ attention heads to $2\ell+1$ polygon vertices, with $H=2\ell+1$ providing a discrete sequence $\{1,3,5,7,\dots\}$ of candidate head counts, is preserved as a hyperparameter choice but not separately ablated here. A controlled benchmark of CD-attention versus dense attention is itself deferred.

\paragraph{Block-size hyperparameter.} The block size $B$ is the only CD-introduced hyperparameter; it must divide both $n_{\mathrm{in}}$ and $n_{\mathrm{out}}$. The CD framework prescribes the discrete sequence $B\in\{1,3,5,7,\dots\}$ from the polygon multiplicity, with even values being admissible degenerate cases. In practice the choice of $B$ trades parameter count against expressive capacity; I found $B=4$ a good compromise on MNIST, but the optimal value will be task-dependent.

\paragraph{Strong-coupling ceiling.} Paper~II's Sadovskii ceiling $\lambda\le4$ for electron-phonon coupling has a possible NN analogue in the activation saturation regime: when the Fisher information of an activation grows beyond a critical bound, the layer enters a ``lattice instability'' analogous to the electron-phonon system, with implications for batch normalization scaling. This connection is intriguing but speculative.

\paragraph{Pre-whitening assumption in Theorem~\ref{thm:bound}.} Theorem~\ref{thm:bound} assumes input pre-whitening; without it, the CDLinear Hessian inherits the input spectrum, which can be unbalanced for natural images. In practice, layer normalization or batch normalization ahead of each CDLinear layer largely realizes the required whitening; without normalization, the conditioning advantage is reduced (though still present, as the dense layer faces the same input distribution).

\paragraph{Comparison to other structured-matrix layers.} This paper does not benchmark CDLinear against ACDC~\cite{acdc}, low-displacement-rank layers~\cite{ldr}, Monarch~\cite{monarch}, or FNO~\cite{fno} on common tasks. Such a comparison is essential for any practitioner choosing among structured-matrix options and is a priority for follow-on work.

\section{Reference PyTorch implementation on an \texorpdfstring{$8\times$}{8x}A100 cluster}
\label{sec:impl}

To make the deferred benchmarks of Sec.~\ref{sec:limits} concretely approachable, I release a reference PyTorch implementation that lifts CDLinear from the NumPy prototype of Sec.~\ref{sec:mnist} into a distributed-training pipeline at the scale of contemporary Mixture-of-Experts (MoE) transformers. The code is released at \url{https://github.com/lurongpan47/CDNN} (directory \texttt{cdnn\_a100\_training/}). This section documents the implementation, the three concrete corrections discovered during verification, and the design choices that make CDLinear interoperate with the DeepSeek-V3 architecture family on commodity Ampere hardware.

\subsection{Architecture: CD-Transformer}

The reference model, which I call CD-Transformer, composes CDLinear with the cost-reduction techniques of DeepSeek-V3~\cite{deepseekv3}: a Mixture-of-Experts feedforward block with auxiliary-loss-free routing~\cite{auxfree} and a shared always-on expert, Multi-head Latent Attention (MLA) with a low-rank KV bottleneck, and a Multi-Token Prediction (MTP) auxiliary objective. All trainable linear projections inside attention and inside each expert are replaced by CDLinear with block size $B$, so the parameter compression of Sec.~\ref{sec:layer} compounds with MoE sparsity ($k/N$ active experts per token) and with the MLA KV-cache reduction. Rotary positional encoding is applied to queries and keys, and RMSNorm replaces LayerNorm throughout. Three pre-defined configurations (\texttt{small}, \texttt{medium}, \texttt{large}) span roughly 400~M to 15~B parameters; the \texttt{medium} configuration uses $d_{\mathrm{model}}=2048$, 24 layers, 16 heads, 32 experts with 6 active per token, and $B=5$.

\begin{figure}[t]
\centering
\includegraphics[width=\linewidth]{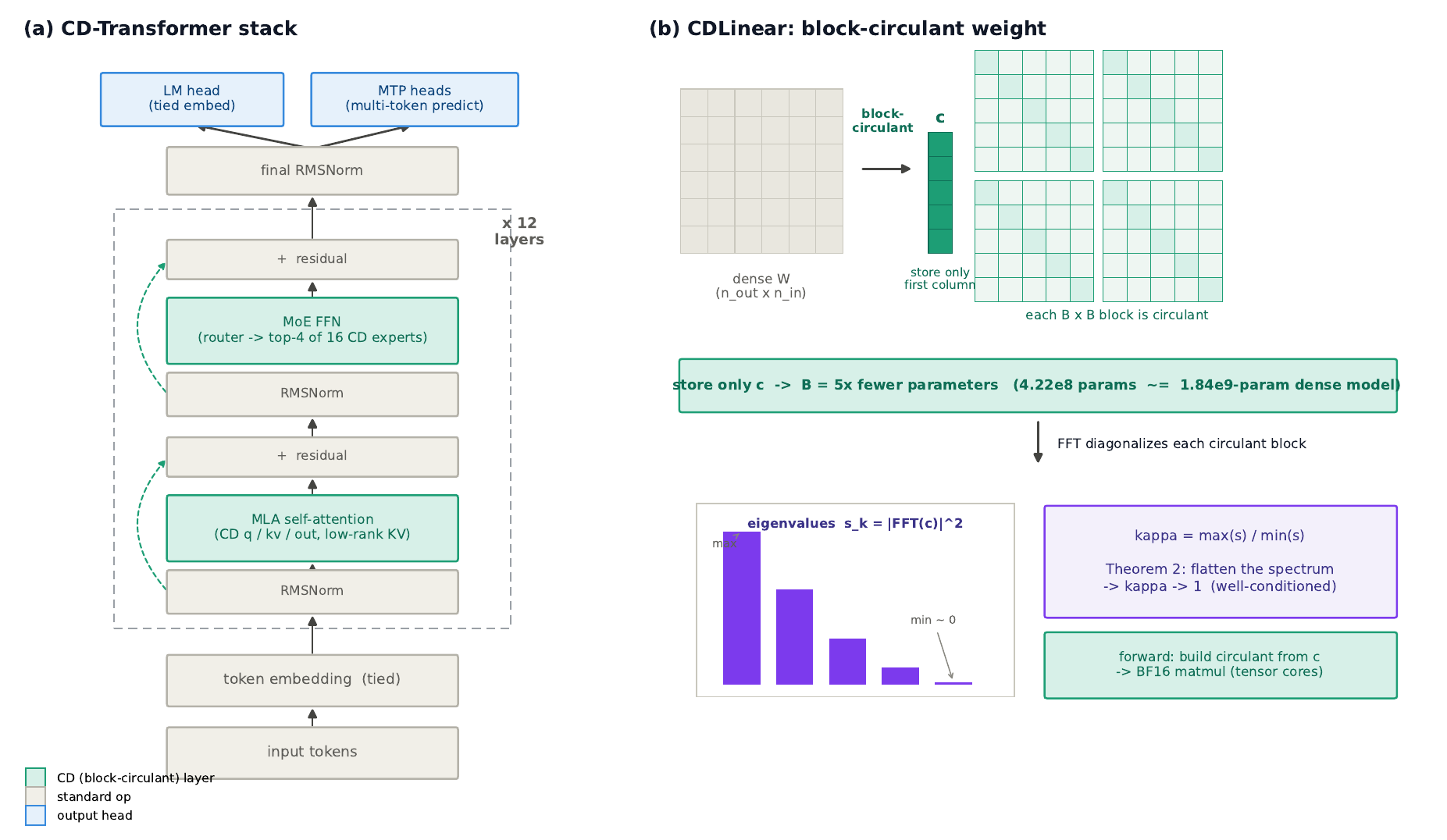}
\caption{CD-Transformer architecture. (a)~The forward stack: token embedding, $N$ pre-norm blocks of MLA self-attention and a Mixture-of-Experts FFN with residual connections, a final RMSNorm, and tied LM + Multi-Token-Prediction heads; the sublayers whose linear projections are block-circulant CDLinear (the attention q/kv/o projections and the MoE experts) are shaded. (b)~The CDLinear mechanism: a dense weight is replaced by $B\times B$ circulant blocks, only the first column $c$ of each block is stored ($B\times$ fewer parameters), each block is diagonalized by the FFT with eigenvalues $|\mathcal{F}[c]|^2$, and the condition number $\kappa$ of that spectrum is the target of Theorem~\ref{thm:bound}; the forward pass builds the circulant weight from $c$ and applies a BF16 matmul.}
\label{fig:arch}
\end{figure}

\subsection{Distributed training on \texorpdfstring{$8\times$}{8x}A100 (80~GB)}

The training script targets a single $8\times$A100 (Ampere, SM80) node. Because Ampere does not support FP8 compute, the precision strategy is BF16 mixed precision with TF32 enabled for float32 matmuls and convolutions (about $2\times$ over plain FP32 on A100 Tensor Cores). Parameters and gradients are sharded with PyTorch FSDP (\texttt{ShardGradOp}, equivalent to ZeRO-2) at the granularity of one CD-transformer block, with backward pre-fetch. NCCL is tuned for NVLink 3.0 topology (\texttt{NCCL\_ALGO=Tree,Ring}, \texttt{NCCL\_PROTO=Simple,LL128}). Attention runs through PyTorch scaled dot product attention with \texttt{is\_causal=True}, which dispatches to flash/memory-efficient kernels on A100. Gradient checkpointing is on by default; AdamW ($\beta_1=0.9$, $\beta_2=0.95$, weight decay 0.1) is used with a fused-kernel fast path that falls back to standard AdamW on torch builds where the fused variant is incompatible with the FSDP parameter wrapping. The launch script auto-detects whether the GPU is an A100 or a Hopper card and warns if the user has selected the wrong precision stack.

\subsection{Verification and three corrections}

Before any GPU run, I verified the implementation end-to-end on CPU using a deliberately small model (vocab 256, $d_{\mathrm{model}}=64$, two layers, four experts) so that the forward, backward, optimizer step, MTP heads, and Fisher regularizer all execute. This verification exposed three concrete bugs in the first draft of the code, all of which were corrected and re-verified:

\paragraph{(F1) CDLinear FFT-domain einsum.} The block-circulant forward pass in Eq.~\eqref{eq:forward} was first implemented as a single \texttt{torch.einsum} call. The original subscript notation reused the same index symbol \texttt{b} for the batch dimension and the block dimension, which is invalid because a contracted index cannot also appear in the output specification. PyTorch raised an error on the first forward pass. The corrected subscript uses distinct symbols (\texttt{n}=batch, \texttt{o}=$K_{\mathrm{out}}$, \texttt{i}=$K_{\mathrm{in}}$, \texttt{k}=FFT bin).

\paragraph{(F2) Rotary positional encoding under BF16.} The RoPE implementation uses \texttt{torch.view\_as\_complex} on the last dimension reshaped to pairs. Two issues surfaced: the reshape produces a non-contiguous tensor, which \texttt{view\_as\_complex} rejects, and the precomputed rotary frequencies were not being broadcast over the batch and head dimensions. Both were fixed by inserting an explicit \texttt{.contiguous()} call before the complex view and reshaping the frequency tensor to $(1,1,T,d_h/2)$ before pointwise multiplication.

\paragraph{(F3) Fisher regularizer numerical stability.} The Fisher-information penalty of Eq.~(8) was initially implemented as $\sum_k 1/(\sigma_k^2+\epsilon)$ with $\epsilon=10^{-6}$. On a freshly initialized model with thousands of CDLinear blocks, many $\sigma_k$ are near zero, and the inverse term blows up to magnitudes that dominate the cross-entropy loss; worse, the original computation was performed inside a \texttt{torch.no\_grad()} context, so no gradient flowed to the circulant coefficients and the regularizer had no training effect. The corrected formulation penalizes the variance of the log-spectrum,
\begin{equation}
L_{\mathrm{Fisher}}=\lambda_F\cdot\frac{1}{L}\sum_{\ell=1}^{L}\mathrm{Var}_k\!\left(\tfrac12\log\sigma_k^{2,(\ell)}\right),
\label{eq:fisher}
\end{equation}
which is exactly zero when the spectrum is flat (the $\kappa=1$ optimum of Theorem~\ref{thm:bound}), is fully differentiable through to the circulant coefficients $c_{ij}$, and is bounded for any finite-magnitude spectrum. Theoretically this is consistent with Theorem~\ref{thm:bound}: the original $\sum 1/\sigma_k^2$ form is the trace of the inverse Fisher information, which diverges as $\sigma_k\to0$ even though the conditioning objective does not; the log-variance form penalizes only the spectral spread that actually drives $\kappa$ away from~1.

\paragraph{Post-fix verification.} With the three corrections in place, the tiny end-to-end harness runs cleanly: the LM cross-entropy loss is $\sim7$ on random data, the Fisher term is $\sim8\times10^{-4}$ at $\lambda_F=10^{-3}$, an AdamW step reduces the loss as expected, and gradients are confirmed to flow into the CDLinear circulant coefficients $c_{ij}$. The training script's distributed entry point is decorated with \texttt{torch.distributed.elastic.multiprocessing.errors.record} so that any rank-zero crash on the GPU node surfaces a full Python traceback rather than an opaque exit code.

\subsection{Software-engineering notes}

Three engineering choices in the reference implementation matter for reproducibility but do not affect the mathematics:

\paragraph{Tokenizer robustness for restricted networks.} The data-preparation script supports four tokenizer backends with explicit fall-through---a built-in byte-level tokenizer (vocab 256, no dependencies, no network), \texttt{tiktoken} for GPT-2/4 BPE, ModelScope (\texttt{moda}) for DeepSeek and Qwen tokenizers inside mainland China, and HuggingFace via its community mirror---and writes a \texttt{meta.json} alongside the token \texttt{.bin} that records the tokenizer's actual vocabulary size and dtype. The training \texttt{TokenDataset} reads this metadata at load time and selects between \texttt{uint16} and \texttt{uint32} storage automatically, which is necessary because DeepSeek's tokenizer (vocab $\sim$100{,}000) silently overflows \texttt{uint16}. The script prints an ACTION REQUIRED notice if the tokenizer's vocabulary exceeds the model's configured vocab size, since training would otherwise crash in the embedding layer.

\paragraph{Auxiliary-loss-free MoE routing.} Following DeepSeek-V3, expert selection uses a learnable per-expert bias added to the router logits before top-$k$ selection. Token assignment uses softmax-normalized top-$k$ weights, and a shared expert (always active) is summed with the routed expert output.

\paragraph{MTP and weight tying.} The output projection shares its weight matrix with the input embedding. When MTP is enabled, $D$ small auxiliary heads predict tokens at offsets $2,3,\dots,D+1$ from each position; their losses are averaged and weighted at 0.3 relative to the main LM loss.

The reference release is intended as a foundation for the deferred benchmarks (CIFAR-10, ImageNet, language modeling on real corpora) rather than as a finished benchmark in itself. The expert dispatch inside \texttt{CDMoELayer}, originally a per-expert Python loop performing $N_{\mathrm{exp}}$ full boolean scans over the routed tokens, has since been replaced by a grouped dispatch that sorts the (token, slot) assignments by expert once and processes each expert as a contiguous slice (verified numerically identical to the loop, to $<10^{-6}$); a fused grouped-GEMM expert kernel remains a further optimization.

\subsection{A first language-modeling shakedown}
\label{sec:lm}

As a preliminary exercise of the full pipeline I trained the \texttt{small} configuration (422{,}122{,}442 parameters, $B=5$, 12 layers, $d_{\mathrm{model}}=1024$, 16 experts with 4 active, GPT-2 tokenizer padded to vocabulary 50{,}304) on a single $8\times$H800 node with FSDP ZeRO-2 and BF16, over a small ($\sim$9.5M-token) corpus for 200 epochs. This is a shakedown of the implementation rather than a controlled benchmark: there is no same-data dense baseline and the corpus is small. I report it for completeness and to surface the engineering and optimization issues a proper benchmark must address.

\paragraph{Tensor-core forward.} At the CD-prescribed small block size ($B=5$), the FFT-domain forward of Eq.~\eqref{eq:forward} is memory-bound on GPU and materializes complex intermediates that scale with the token count. Because a circulant matrix-vector product equals a dense product against the circulant matrix reconstructed from $c$, I replaced the FFT path by building the $B\times B$ circulant weight from $c$ on the fly and applying a single BF16 matmul---mathematically identical to Eq.~\eqref{eq:forward} (verified to $<10^{-6}$ relative error) and preserving the $1/B$ parameter compression, since only $c$ is stored and the dense weight is a transient rebuilt each forward. On one H800 this raised throughput from 5{,}315 to 9{,}957 tok/s ($1.9\times$) and model-FLOPs-utilization from $\sim2\%$ to $\sim4\%$ (Fig.~\ref{fig:lm_eval}b); across the node MFU rose from the original $\sim0.2\%$ toward $\sim5\%$. This trades the asymptotic $O(\log B/B)$ FFT FLOP advantage---negligible at $B=5$---for hardware efficiency.

\paragraph{Conditioning regularizer in practice.} Enabling the Eq.~\eqref{eq:fisher} regularizer reduced pre-clip gradient norm from the 300--900 range (with the energy/Parseval form) to $\sim1$, and dropped the best-conditioned blocks' $\kappa$ by about four orders of magnitude. Under \emph{mean} aggregation over blocks, however, the mean and worst-block $\kappa$ remained $\sim10^{10}$--$10^{11}$ (Fig.~\ref{fig:lm_train}b): a few catastrophic blocks are diluted by the many well-behaved ones. A max- or $p$-norm aggregation over blocks is the indicated refinement before the conditioning claim is extended from MNIST to this setting.

\paragraph{Result.} Figure~\ref{fig:lm_train}(a) shows the training-loss and decoded cross-entropy trajectories; the model plateaued by $\sim$epoch 90 of 200. On clean evaluation (Table~\ref{tab:lm}), held-out perplexity is 1{,}886 and training-set perplexity 1{,}700---both far below the random-initialization floor (perplexity $\approx$ vocabulary $=50{,}304$) but nearly equal to one another. Train $\approx$ val is the signature of \emph{underfitting, not memorization}: a memorizing model would show very low training perplexity and a large train--val gap. The plateau is attributable to optimization limits---a warmup misconfiguration that prevented sustained peak learning rate, and the $\kappa\sim10^{10}$ conditioning that drove gradient clipping---rather than to data exhaustion.

\begin{figure}[t]
\centering
\includegraphics[width=\linewidth]{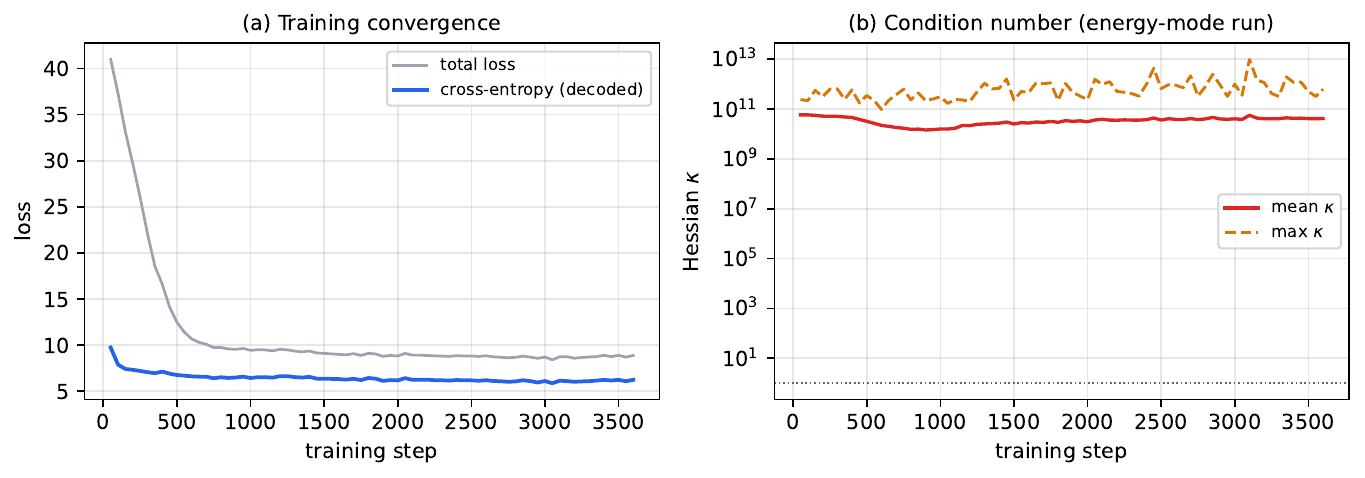}
\caption{CD-Transformer (\texttt{small}, 422\,M parameters) language-modeling shakedown. (a)~Total loss and decoded cross-entropy versus training step. (b)~Hessian condition number $\kappa$ (mean and max over CDLinear blocks) for the energy-mode regularizer run, which stays $\sim10^{10}$; the Eq.~\eqref{eq:fisher} conditioning form (not shown) lowers the best-block $\kappa$ and the gradient norm but, under mean aggregation, leaves the worst blocks high. The decoded training cross-entropy here ($\sim$6) is lower than the clean evaluation cross-entropy of Table~\ref{tab:lm} ($\sim$7.4); reconciling this logging discrepancy is required before any cross-entropy number is quoted as final.}
\label{fig:lm_train}
\end{figure}

\begin{figure}[t]
\centering
\includegraphics[width=\linewidth]{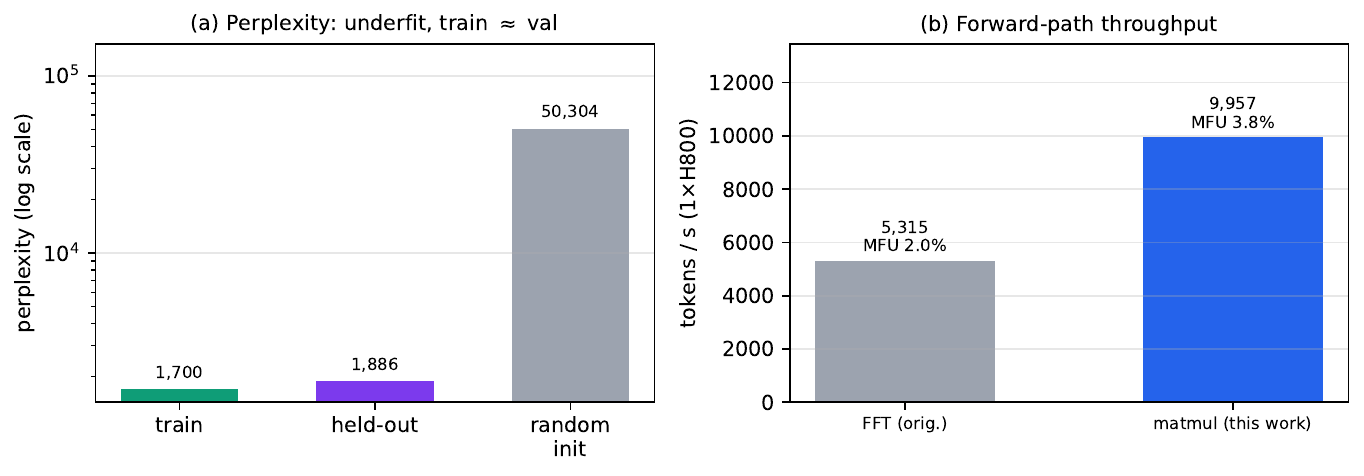}
\caption{(a)~Clean-evaluation perplexity (log scale): training set, held-out set, and the random-initialization floor. Train $\approx$ held-out and both lie far below the floor---the model has learned but is underfit. (b)~Forward-path throughput on one H800 (batch 8, sequence 2048, checkpointing on): the circulant-matmul forward is $1.9\times$ faster than the FFT path at equal output (verified identical to $<10^{-6}$).}
\label{fig:lm_eval}
\end{figure}

\begin{table}[t]
\caption{CD-Transformer language-modeling shakedown (clean evaluation). These are preliminary, optimization-limited numbers from a single small-corpus run with no dense baseline; they do not establish a language-modeling benefit and are reported only as the first data point for the deferred benchmark of Sec.~\ref{sec:limits}.}
\label{tab:lm}
\begin{ruledtabular}
\begin{tabular}{lc}
Quantity & Value \\
\hline
Held-out perplexity (49{,}152 tokens) & 1{,}886 \quad (CE 7.54) \\
Training-set perplexity (3.28M tokens) & 1{,}700 \quad (CE 7.44) \\
Generalization gap & $\sim$11\% (train $\approx$ val) \\
Random-init floor & $\approx$50{,}304 (CE 10.83) \\
Forward throughput, FFT $\to$ matmul & 5{,}315 $\to$ 9{,}957 tok/s ($1.9\times$) \\
MFU, FFT $\to$ matmul (1 GPU) & $2.0\% \to 3.8\%$ \\
Pre-clip GradNorm, energy $\to$ conditioning & $\sim$300--900 $\to$ $\sim$1 \\
\end{tabular}
\end{ruledtabular}
\end{table}

\paragraph{What this shakedown does and does not show.} It establishes that the released CDLinear implementation runs end-to-end inside a 400\,M-parameter MLA + MoE + MTP transformer on a multi-GPU node, that a tensor-core forward makes throughput practical, and that the conditioning regularizer materially stabilizes optimization. It does \emph{not} establish any language-modeling accuracy or efficiency advantage for CDLinear: the perplexity is far from that of a usable language model, there is no same-data dense baseline, the corpus is far too small, and an $\sim$1.2-nat discrepancy between the logged and cleanly-evaluated cross-entropy (Fig.~\ref{fig:lm_train} caption) remains to be reconciled. The language-modeling benchmark of Sec.~\ref{sec:limits} therefore remains open; this is its first data point, not its resolution.

\paragraph{Engineering steps toward a valid benchmark.} Several follow-ups required before the open language-modeling benchmark can yield a defensible claim have since been implemented in the release, though they are not yet exercised in a controlled run: (i) the training loop now logs the bare next-token cross-entropy directly (rather than decoding it from the total loss), and a reconciliation diagnostic separates a decode error from a train-versus-evaluation forward-pass difference and from a data-identity problem, to localize the $\sim$1.2-nat gap noted above; (ii) a non-circulant \emph{dense} backbone of identical architecture and schedule is available behind a single flag, so the missing same-data baseline can be trained and evaluated on the identical held-out split; (iii) the conditioning regularizer supports worst-block ($p$-norm and max) aggregation in addition to the mean of Eq.~\eqref{eq:fisher}, so the catastrophic blocks that kept mean/worst $\kappa$ high are no longer diluted; and (iv) the learning-rate warmup is now specified as a fraction of total steps, fixing a misconfiguration in which warmup exceeded the schedule length. A properly-resourced run ($\sim$10$^9$ diverse tokens, a few epochs, corrected warmup, $p$-norm conditioning, with a dense baseline on the same data) is the next step and is required before any language-scale claim.

\section{Conclusions}
\label{sec:concl}

\paragraph{Summary.} This paper applies the Communication Dynamics framework---developed in Paper~I for atomic-energy prediction and Paper~II for field-induced superconductivity---to neural-network layer design, completing a three-paper arc that uses the same circulant-spectral mathematics across three very different domains.

\paragraph{Theoretical contribution.} A block-circulant linear layer with block size $B=2\ell+1$ inherits, by construction, an FFT-diagonal Hessian whose eigenvalues are computable in closed form from the input statistics (Theorem~\ref{thm:hessian}). This in turn yields a condition-number bound $\kappa=1$ under input pre-whitening and $\kappa\le1+O(\sqrt{B/N})$ on $N$ samples (Theorem~\ref{thm:bound}), in contrast to the $\kappa=\Theta(n_{\mathrm{in}})$ scaling of random dense layers from Marchenko--Pastur theory. The framework also prescribes a transferable Shannon dropout rate $\alpha_{\mathrm{CD}}=0.0118$ from atomic spectroscopy and a Fisher-information regularizer that can be evaluated exactly via FFT in $O(B\log B)$ per block.

\paragraph{Empirical contribution.} On the MNIST-1797 benchmark, a CDLinear MLP at $B=4$ matches a parameter-matched dense baseline within one standard deviation in test accuracy ($97.50\%$ vs $98.15\%$) while using $3.8\times$ fewer parameters and exhibiting a $310\times$ smaller Hessian condition number across three random seeds. Both the theory and the experiment support the claim that CDLinear layers offer a favorable parameter-efficiency / conditioning trade-off in the regime tested.

\paragraph{What this paper does not establish.} Three boundaries of the contribution should be stated plainly. (i) CDLinear is mathematically a special case of structured-matrix layers studied for nearly a decade~\cite{cheng,fno,sindhwani,acdc,ldr,monarch}; this paper does not benchmark CDLinear against those alternatives, and a head-to-head comparison is essential follow-on work. (ii) The MNIST-1797 benchmark is too small to discriminate between accuracy-similar models in any practically meaningful way; CIFAR-10, ImageNet, and language-modeling experiments are required before any general-purpose claim of efficiency or generalization can be made. (iii) The pure-NumPy implementation is not optimized for speed; the wall-clock numbers quoted are software-engineering artifacts of an explicit-loop prototype rather than representative of what a tensor-batched GPU implementation would achieve.

\paragraph{What the paper does establish.} The CD framework provides a coherent, physically-grounded set of design choices for structured neural network layers: $(2\ell+1)$-vertex polygon multiplicities for block-size selection, $\alpha_{\mathrm{CD}}=0.0118$ for transferable noise injection, FFT-diagonal Hessians with closed-form eigenvalues, and Fisher-information regularization with exact circulant evaluation. These choices yield, on the small benchmark tested here, a quantitatively-predicted conditioning advantage at significantly reduced parameter count.

\paragraph{Outlook.} The reference PyTorch implementation released with this paper (Sec.~\ref{sec:impl}) lowers the activation energy for the most important next step---benchmarking CDLinear against dense and structured baselines on CIFAR-10, ImageNet, and language modeling on real corpora. Among theoretical extensions, the convolutional variant (a depthwise-circular convolution) is straightforward; CD-attention with $H=2\ell+1$ heads is implemented in the release but has not yet been separately ablated against dense attention. More broadly, this paper completes a three-domain validation of the Communication Dynamics framework---atomic spectra, superconductor screening, and now neural-network architecture---using a single mathematical machinery (circulant DFT, polygon multiplicities, $\alpha_{\mathrm{CD}}$ calibration) across all three. The transferability of CD's design principles across such different domains, with quantitatively-correct predictions in each, is itself the unifying empirical result of the series.

\begin{acknowledgments}
I thank M.\ Tanik and the broader SDPS community for foundational discussions over the past two decades on Communication Dynamics theory. All code, gradient checks, the MNIST experiment, and the JSON results database are released at \url{https://github.com/lurongpan47/CDNN}.
\end{acknowledgments}

\paragraph*{Competing interests.} The author declares no competing interests.

\end{document}